 \documentclass[onecolumn, draftclsnofoot,11pt]{IEEEtran}
\IEEEoverridecommandlockouts

\usepackage{booktabs}

\usepackage{amsthm}

\usepackage{times}
\usepackage{algorithm}      
\usepackage[noend]{algpseudocode} 

\usepackage{amsmath}
\usepackage{amsfonts}

\usepackage{mathrsfs}
\usepackage{tikz}
\usepackage[utf8]{inputenc}
\usepackage{amsfonts} 
\usepackage{epstopdf}
\usepackage{mathtools}
\usepackage{graphicx}
\usepackage{bbm}
\usepackage{enumerate}
\usepackage{epsf,verbatim,amssymb,array,cite,multicol,multirow}  
\usepackage{psfrag,bm,xspace}
\usepackage{hhline}
\usepackage{mathabx}
\usepackage{xstring}
\usepackage{ifthen}
\newtheorem{theorem}{Theorem}
\newtheorem{lem}{Lemma}
\newtheorem{corollary}{Corollary}
\newtheorem{fact}{Fact}
\theoremstyle{definition}
\newtheorem{definition}{Definition}

\makeatletter
\def\old@comma{,}
\catcode`\,=13
\def,{%
  \ifmmode%
    \old@comma\discretionary{}{}{}%
  \else%
    \old@comma%
  \fi%
}
\makeatother

\usepackage{xcolor}
\definecolor{darkblue}{rgb}{0.1,0.1,0.8}
\definecolor{brickred}{rgb}{0.8, 0.25, 0.33}
\definecolor{britishracinggreen}{rgb}{0.0, 0.26, 0.15}
\definecolor{calpolypomonagreen}{rgb}{0.12, 0.3, 0.17}
\definecolor{ao(english)}{rgb}{0.0, 0.5, 0.0}
	\definecolor{cadmiumgreen}{rgb}{0.0, 0.42, 0.24}
\definecolor{burgundy}{rgb}{0.5, 0.0, 0.13}
\usepackage{etoolbox}

\providetoggle{Blue_revision}
\settoggle{Blue_revision}{true}

\providetoggle{Track}
\settoggle{Track}{true}
\newcommand{\addv}[3]{%
	\iftoggle{Track}{%
    	\IfEqCase{#1}{%
       	 	{a}{\ifthenelse{\equal{#2}{ON}}{{\color{cadmiumgreen}#3}}{#3}}%
        	{b}{\ifthenelse{\equal{#2}{ON}}{{\color{brickred}#3}}{#3}}%
       		{c}{\ifthenelse{\equal{#2}{ON}}{{\color{burgundy}#3}}{#3}}%
    	}[\PackageError{tree}{Undefined option to tree: #1}{}]%
	}{#3}%
}

 \usepackage{hyperref}
\usepackage{xcolor}
\definecolor{DarkGreen}{rgb}{0,0.6,0}

\hypersetup{
colorlinks=true, %
  pdfstartview={FitH},
    linkcolor=red,
    citecolor=blue,
    urlcolor={blue!80!black}
}

\usepackage{graphicx}

\newcounter{relctr} 
\everydisplay\expandafter{\the\everydisplay\setcounter{relctr}{0}} 

\AtBeginDocument{} 

\global\long\def\RR{\mathbb{R}}

\global\long\def\NN{\mathbb{N}}

\global\long\def\EE{\mathbb{E}}
\global\long\def\PP{\mathbb{P}}

\global\long\def\SS{\mathbb{S}}
\global\long\def\11{\mathbbm{1}}

\newcommand{\bfx}{\mathbf{x}}

\newcommand{\bfX}{\mathbf{X}}

\global\long\def\+{\oplus}

\newcommand\pmm{\{-1,1\}}

\newcommand{\prob}[1]{\PP\Big\{  #1 \Big\} }
\def\<{\langle}
\def\>{\rangle}

\newcommand{\sign}{\mathsf{sign}}
\newcommand{\OR}{\mathsf{OR}}
\newcommand{\AND}{\mathsf{AND}}

\newcommand{\abs}[1]{\lvert#1\rvert}
\newcommand{\norm}[1]{\lVert#1\rVert}

\DeclareMathOperator*{\argmin}{arg\,min}

\usepackage{stackengine}
\def\deq{\mathrel{\ensurestackMath{\stackon[1pt]{=}{\scriptstyle\Delta}}}}

\def\fS{{f}_{\mathcal{S}}}

\def\hJ{h_{\mathcal{J}}}

\def\ps{\psi_\mathcal{S}}
\def\pS{\ps}

\def\festS{\hat{f}_{ S} }

\def\fbarS{\bar{f}_{ S} }

\def\pestS{\widehat{\psi}_{\mathcal{S}}}

\def\LD{\mathcal{L}_2(D)}
\def\LDemp{\mathcal{L}_2(\Demp)}

\def\L2{\LD}

\newcommand{\optfont}[1]{\mathsf{#1}}
\newcommand{\Popt}{\optfont{P}_{opt}}
\newcommand{\Pek}{\Popt}

\def\E_mu{\mathcal{E}_{\mu}(\epsilon')}

\def\Ltwo{\mathcal{L}_2}
\def\degree{\operatorname{degree}}
\def\Pemp{\hat{\PP}}
\def\Demp{\hat{D}}
\def\Dx{D_{\bfx}}

\def\Pekemp{\widehat{\optfont{P}}_{opt}}
\def\Pk{\mathcal{P}_k}

\def\Piest{\hat{\Pi}_Y}
\def\Piy{\Pi_Y}
\def\Pibar{\bar{\Pi}_Y}
\def\PiJ{\Pi_Y^{\mathcal{J}}}

\def\aS{a_{\mathcal{S}}}

\usepackage[nolist,nohyperlinks]{acronym}

\newacro{ptp}[PtP]{Point-to-Point}
\newacro{iid}[i.i.d.]{independent and identically distributed} 
\newacro{IID}[i.i.d.]{independent and identically distributed} 
\newacro{UFFS}[UFFS]{Unsupervised Fourier Feature Selection}
\newacro{SFFS}[SFFS]{Supervised Fourier Feature Selection}
\newacro{LS}[LS]{Laplacian Score}
\newacro{MAE}[MAE]{mean absolute error}
\newacro{MSE}[MSE]{mean square error}
\newacro{PAC}[PAC]{probably approximately correct}
\newacro{VC}[VC]{Vapnik–Chervonenkis}
\newacro{ERM}[ERM]{Empirical Risk Minimization}
\newacro{SVM}[SVM]{support-vector machine}
\settoggle{Track}{false}

\newcommand{\citep}{\cite}
\newcommand{\citet}{\cite}

\begin{document}
\title{On Agnostic PAC Learning using $\Ltwo$-polynomial Regression and Fourier-based Algorithms }

	\author{Mohsen Heidari and Wojciech Szpankowski,\\
     Department of Computer Science, Purdue University,\\
\tt \{mheidari, szpan\}@purdue.edu}

\maketitle
\begin{abstract}
We develop a framework using Hilbert spaces as a proxy to analyze PAC learning problems with structural properties. We consider a joint Hilbert space incorporating the relation between the true label and the predictor under a joint distribution $D$. We demonstrate that agnostic PAC learning with 0-1 loss is equivalent to an optimization in the Hilbert space domain. With our model, we revisit the PAC learning problem using methods based on \textit{least-squares} such as $\Ltwo$ polynomial regression and Linial's low-degree algorithm. We study learning with respect to several hypothesis classes such as half-spaces and  polynomial-approximated classes (i.e., functions approximated by a fixed-degree polynomial). We prove that (under some distributional assumptions) such methods obtain generalization error up to $2\Popt$ with $\Popt$ being the optimal error of the class. Hence, we show the tightest bound on generalization error when $\Popt\leq 0.2$. 
\end{abstract}


\section{Introduction}
We study binary classification using polynomial regression from the agnostic {PAC} learning perspective \citep{Valiant1984,Kearns1994}. In this problem, multiple training instances are generated IID according to an underlying distribution $D$ on the feature-label sets  $\mathcal{X}\times \pmm$. In addition, we are given a hypothesis class with respect to which the learning process takes place. If $\Pek$ is the minimum error attained using the given class, then the objective of the learning algorithm is to output, with high probability, a classifier whose generalization error is not greater than $\Pek+\epsilon$.

To gain computational efficiency or analytical tractability, many conventional learning methods such as \ac{SVM} rely on intermediate loss functions other than the natural $0-1$ loss. Square loss is an example that is a basis for  $\Ltwo$-polynomial regression or another variant of SVM known as LS-SVM \citep{Suykens1999}. The well-known ``low-degree" algorithm \citep{linial1993constant} is also known to be in this category of algorithms \citep{Kalai2005}. Such methods  have been analyzed for many PAC learning problems. 
Under the \textit{realizability} assumption where $\Pek=0$, the  $\Ltwo$-polynomial regression and the low-degree algorithm are PAC learners for a variety of hypothesis classes \citep{mossel2004learning,Mossel_ODonnell,blais2010polynomial}.
Under the agnostic setting where $\Pek>0$, the current results are not that promising. The best known results for $\Ltwo$-polynomial regression (and the low-degree algorithm under the uniform distribution) are  $8\Pek$ and $\frac{1}{4} + \Pek(1-\Pek)$ for classes such as half-spaces or polynomial-approximated classes \citep{Kalai2005,Kearns1994}.    

In this paper, we develop a framework using Hilbert spaces as a proxy to analyze  such problems. We consider a joint Hilbert space incorporating the relation between the true label and the predictor under the joint distribution $D$. This is unlike conventional analysis using Hilbert spaces that focus only on the predictors with marginal $\Dx$ on the features. As a byproduct, we improve the above mentioned bounds and show that the generalization error of $\Ltwo$-polynomial regression and the low-degree algorithm is less than $2\Pek$. This bound the improves upon the previous bounds when $\Popt\leq 0.2$. 
We show that methods based on square loss are suitable for learning classes with appropriate geometrical properties. 

\subsection{Our approach} 
We develop our framework by constructing two Hilbert spaces one with respect to the true underlying distribution $D$ and the other with respect to the empirical one. The first one is $\mathcal{L}_2(D)$, that is all real-valued functions $f$ on $\mathcal{X}\times \mathcal{Y}$ such that  $\EE[f(\bfX,Y)^2]<\infty$. The second one is $\mathcal{L}_2(\Demp)$  with $\Demp$ being the empirical distribution of the training set.   With this formulation, the true label $Y$ and the training labels are understood as a member of these spaces. With this formulation, the generalization error of any classifier $c$ equals $\frac{1}{4}\norm{Y-c}_{2, D}^2$. Similarly, when the distance is calculated in the second Hilbert space, we obtain a characterization of the empirical error.  Hence, minimizing the generalization (or empirical) error is equivalent to minimizing the distance between $Y$ and the classifier $c$ in the first (or second) Hilbert space.  
We argue that the mentioned hypothesis classes have appropriate structures using that allows us to drive lower bounds on its minimum error $\Pek$. For instance, given $k$, the polynomial-approximated class is characterized by the subspace of  $\mathcal{L}_2(D)$ spanned by polynomials of degree up to $k$. With this structure, finding $\Pek$ is equivalent to finding the minimum distance of $Y$ to the subspace spanned by polynomials of degree up to $k$. As for the learning algorithms, we argue the low-degree algorithm and $\Ltwo$-polynomial regression have suitable structures using which we drive our upper bounds on their generalization errors. For instance, in the case of $\Ltwo$ polynomial regression, the  error of any classifier of the form $\sign[p(x)-\theta]$, with $\theta$ chosen appropriately, is bounded from above by $\frac{1}{2}\norm{Y-p}_2^2$. Hence, minimizing the squares-loss as in $\Ltwo$-regression yields an error less  than $2\Pek$.


\subsection{Summary of the Results}
In this work, we first present a more general version of the low-degree algorithm incorporating non-uniform but product probability distributions. We refer to this generalization as Fourier algorithm.  
With our framework, we study learning with respect to three well-known hypothesis classes.  The first class is half-spaces consisting of all the Boolean-valued functions of the form $c(\bfx)=\sign[\sum_{j=1}^d w_j x_j-\theta]$. The second class is called polynomial-approximated functions. Given a positive integer $k$ and $\epsilon>0$, it consists of Boolean-valued functions that are approximated by a degree $k$ polynomial with square error up to $\epsilon^2$.  The thirst class is a generalization of the second. 
 We use our framework to analyze learning these hypothesis classes using $\Ltwo$-polynomial regression and the Fourier algorithm. Below, is the summary of our results: 

\textbf{1)} The $\Ltwo$ polynomial regression with degree $k$ outputs a hypothesis $\hat{g}$ whose generalization error has the following properties:

\begin{itemize}
\item For learning polynomial-approximated classes, it is less than $2\Pek +3\epsilon$ (Theorem \ref{thm:L2 polynomials}).
\item For learning half-spaces, when the marginal $\Dx$  is uniform over the unit ball in $\RR^d$, it is less than $2\Pek+3\epsilon$ (Theorem \ref{thm:L2 polynomials half-spaces}).
\item For learning \textit{generalized concentrated classes }, under any distribution, it is less than $2\Pek+\epsilon$ (Theorem \ref{thm:general hilbert class}).
\end{itemize}

\textbf{2)} If the marginal $\Dx$ is a product probability distribution on $\pmm^d$, then with probability $(1-\delta)$, the Fourier algorithm outputs a hypothesis such that its generalization error is less than  $2\Pek+2\epsilon$ for learning  polynomial-approximated classes. 


\subsection{Related Works}
The low-degree algorithm is introduced by \citet{linial1993constant} with PAC learning guarantees under the uniform distribution over $\pmm^d$. This algorithm which is based on the Fourier expansion on the Boolean cube has been used for in various problems  \citep{mossel2004learning,blais2010polynomial,JMLRFourier}. The $\mathcal{L}_2$ polynomial regression along with its $\mathcal{L}_1$ counterpart is introduced by \citep{Kalai2005} for learning with respect to polynomial-approximated classes, $k$-juntas, and half-spaces. Learning with respect to such classes has been studied extensively in the literature\citep{Kalai2005,Klivans2009,Birnbaum2012,Diakonikolas2019}.  Among such classes, learning with respect to half-spaces is the most challenging. In the case of \textit{proper} agnostic PAC learning, where the algorithm's predictor must be a half-space, it is an NP-hard problem \citep{Klivans2014,Guruswami2009}. Even without the \textit{proper} restriction, the problem is NP-hard. That said, under distributional assumptions, polynomial time algorithms are introduced \citep{Kalai2005,Awasthi2017,Daniely2015}. Among them are the \textit{improper} learning algorithms based on regression methods such as $\mathcal{L}_1$ or $\Ltwo$ polynomial regression \citep{Kalai2005,linial1993constant}. In particular, \citet{Kalai2005} proved that $\mathcal{L}_1$ polynomial regression learns a range of hypothesis classes such as half-spaces (under distributional assumptions) and polynomial-approximated classes
\section{Preliminaries}
\noindent\textbf{Notation:} 
The input set is denoted by $\mathcal{X}$ which is a subset of $\RR^d$ for some  positive integer $d$. The output set is denoted by $\mathcal{Y}$ which is a subset of $\RR$. In binary classification $\mathcal{Y}=\pmm$.  
 For shorthand, the random vectors in $\RR^d$ are denoted by $\bfX = (X_1, X_2, ..., X_d)$.  Further, for any ordered subset $\mathcal{J}=\{j_1, j_2, \cdots, j_m  \}$, 
by $X^{\mathcal{J}}$ denote the random vector $(X_{j_1}, X_{j_2}, \cdots, X_{j_m})$. Similarly, by $x^{\mathcal{J}}$ denote the vector  $(x_{j_1}, x_{j_2}, \cdots, x_{j_m})$. For a pair of functions $f,g$ on $\mathcal{X}$, the notation $f\equiv g$ means that $f(x)=g(x)$ for all $x\in\mathcal{X}$.
Lastly, for any natural number $\ell$, the set $\{1,2,\cdots, \ell\}$ is denoted by $[\ell]$.

\subsection{A Hilbert Space Representation}\label{sub:HilbertSP}
We first develop a Hilbert Space formulation for the binary classification problem. 
Let $D$ be a joint probability distribution on the input-output set $\mathcal{X}\times \mathcal{Y}$. In this paper, it is assumed that the marginal $\Dx$ of any joint distribution $D$ on $\mathcal{X}\times \mathcal{Y}$ has finite moments. Consider a Hilbert space of all real-valued  functions $f:\mathcal{X}\times \mathcal{Y}\mapsto \RR$ which are $\LD$, that is $\EE_D[f(\bfX,Y)^2]<\infty$. The inner product between two members $f,g$ is defined as $$\<f,g\>\deq \EE_D[f(\bfX,Y) g(\bfX, Y)].$$  Given any integer $p>0$ and distribution $D$, the $p$-norm of a function $f$ is defined as $$\norm{f}_{p, D}\deq \big(\EE_D[f(\bfX, Y)^p]\big)^{1/p}.$$ Given any training sample $\mathcal{S}=\{(\bfx_i, y_i): i=1,2,...,n\}$, let $\Demp$ denote its empirical distribution, that is a uniform distribution on $\mathcal{S}$ and zero outside of it.   Associated with this distribution, we consider the Hilbert space $\LDemp$ with the inner product and norms defined based on the empirical distribution $\Demp$.  We use this formulation to study the binary classification problem where  $\mathcal{Y}=\pmm$. Therefore, the generalization error  of any predictor $c:\mathcal{X}\mapsto \pmm$  can be written in terms of the inner products as 
\begin{align}\label{eq:error and inner prod}
\PP_D\Big\{Y\neq c(\bfX)\Big\} &= \frac{1}{2}-\frac{1}{2}\<Y,c\>_D = \frac{1}{4}\norm{Y-c}_{2,D}^2,
\end{align}
where, with slight abuse of notation, $Y$ is understood as the mapping $(x,y)\mapsto y$ and $c$ is understood as a mapping on $\mathcal{X}\times \mathcal{Y}$ which depends only on $\mathcal{X}$. Similarly, the empirical error of $c$  is equal to 
$$
 \Pemp_{\Demp}\Big\{Y\neq c(\bfX)\Big\}  	=\frac{1}{2}-\frac{1}{2}\<Y, c\>_{\Demp}=\frac{1}{4}\norm{Y-c}_{2,\Demp}^2.$$

The goal now is to derive bounds on the minimum generalization error when learning with respect to various hypothesis classes. In Section \ref{sub:L2poly PAC} we describe $\Ltwo$-polynomial regression and the Fourier algorithm, in Section \ref{sec:poly approx} we study polynomial-approximated classes, and finally in Section \ref{sec:other classes} we discuss half-spaces, and more general hypothesis classes that have structural properties. 
 
 \section{PAC Learning with $\Ltwo$-Polynomial Regression}\label{sub:L2poly PAC}
We employ a PAC learning algorithm using $\Ltwo$-polynomial regression. Given a training set, the objective of the polynomial regression is to minimize the empirical square loss over all polynomials of degree up to $k$. This process can be implemented by stochastic gradient descent or by solving a linear system of equations. 
 We describe how this polynomial regression can be used for PAC learning. Let $\hat{p}$ be the output of the polynomial regression. The idea  is to shift the polynomial $\hat{p}$ by a threshold $\theta$ and take its sign. This process is demonstrated as Algorithm \ref{alg:L2}. 


\begin{algorithm}
\caption{PAC Learning with $\Ltwo$-Polynomial Regression}
\label{alg:L2}
\hspace*{\algorithmicindent} \textbf{Input:} {Degree parameter $k$, and training samples $\{(\bfx(i), y(i)), i\in [n]\}$.}
\begin{algorithmic}[1]
\State Find a polynomial $\hat{p}$ of degree up to $k$ that minimizes  $$\frac{1}{n}\sum_i \big(y(i)-p(\bfx(i))\big)^2.$$ 
\State Find $\theta\in [-1,1]$ such that the empirical error of $\sign[\hat{p}(\bfx)-\theta]$ is minimized. \\
\Return   $\hat{g}\equiv \sign[\hat{p}-\theta]$.
\end{algorithmic}
\end{algorithm}

\subsection{Fourier-Based Learning Algorithm}\label{sub:Fourier PAC concentrated}
We present another variant of $\Ltwo$ polynomial regression, known as the low-degree (Fourier) algorithm \citep{linial1993constant}. Although this algorithm is more efficient than the polynomial regression, it requires binary input set $\mathcal{X}=\pmm^d$. The low-degree algorithm was originally designed for uniform distribution on the Boolean cube. In this paper, we present a more general version of it for incorporating non-uniform but product probability distributions on $\pmm^d$ \citep{furst1991improved}. 
In this approach, the objective is to find an estimate of the $p^*$ polynomial that minimizes the square loss $\norm{Y-p^*}_{2,D}$ under the true distribution.  This method is based on the Fourier expansion on the Boolean cube \citep{ODonnell2014} and is summarized in the following. 

 Under product probability distribution on $\pmm^d$, any bounded real-valued functions  can be written as 
\begin{align*}
 f(\bfx) = \sum_{\mathcal{S}\subseteq [d]} \fS~\pS(\bfx), 
 \end{align*}
 where $\fS$'s are the Fourier coefficients and calculated as $\fS\deq \<f,\pS\>$ for every subset $\mathcal{S}\subseteq [d]$. Further,  the parity $\pS$ is a monomial defined as 
 \begin{align*}
  \pS(x) = \prod_{j\in \mathcal{S}}\frac{x_j-\mu_j}{\sigma_j},
  \end{align*} 
  with $\mu_j$ and $\sigma_j$ being the mean and standard deviation of the $X_j$, respectively. As the distribution is unknown, these quantities are estimated in the algorithm. 
  
As a result, we can write the Fourier decomposition of the optimal polynomial $p^*$. For that, we have the following statement:

\begin{fact}
Let $D$ be a probability distribution with the marginal $\Dx$ that is a product probability distribution on $\pmm^d$. Then, the optimal polynomial $p^*$ admits the following Fourier decomposition
$$p^*\equiv \sum_{\mathcal{S}\subseteq [d]: |\mathcal{S}|\leq k} \<Y,\pS\>~\pS.$$
\end{fact}
With that decomposition, the idea behind the Fourier algorithm is to compute an empirical estimate of $\<Y,\pS\>$. This is demonstrated as Algorithm \ref{alg:fourier}.
  
\begin{algorithm}
\caption{Fourier-Based Learning}
\label{alg:fourier}
\hspace*{\algorithmicindent} \textbf{Input:} {Training samples $\{(\bfx(i), y(i)), i\in [n]\}$.}
\begin{algorithmic}[1]
\State Compute the empirical mean $\hat{\mu}_j$ and standard deviation $\hat{\sigma}_j$ of each feature. 
\State For every $\mathcal{S}\subseteq [d]$ with $|\mathcal{S}|\leq k$, construct the empirical parity as $\pestS(\bfx)=\prod_{j\in \mathcal{S}}\frac{x_j-\hat{\mu}_j}{\hat{\sigma}_j}$.
\State Compute the empirical Fourier coefficients $\aS$, for every $\mathcal{S}$ with at most $k$ elements, as 
\begin{align*}
\aS = \frac{1}{n}\sum_{i=1}^n y(i) \pestS(\bfx(i)).
\end{align*}

\State Construct and return  the function $\Piest$ as  
{  $$\Piest(\bfx) \deq \sum_{\mathcal{S}: |\mathcal{S}|\leq k }    \aS\pestS(\bfx).$$ }
\end{algorithmic}
\end{algorithm}

In the following lemma  which is proved in Appendix \ref{proof:Pi_Y est}, we derive bounds for estimating the optimal polynomial $p^*$.
\begin{lem}\label{lem:Pi_Y est}
Let $D$ be a probability distribution with the marginal $\Dx$ that is a product probability distribution on $\pmm^d$. Given $\delta \in (0,1)$, with probability at least $(1-\delta)$,  the following inequality holds  
\begin{align}\label{eq:pstar-Piest}
 \norm{p^*- \Piest }_2 \leq O\Big(\sqrt{\frac{  d^k c_k}{(k-1)!n }\log \frac{4d^k}{(k-1)!\delta}}\Big),
 \end{align}
  where $c_k\deq \max_{\mathcal{S}\subseteq [d], |\mathcal{S}|\leq k} \norm{\pS}_{\infty}^2$ and $n$ is the number of samples. 
\end{lem}

\section{Polynomially Approximated Class}\label{sec:poly approx}
In this section, we study agnostic PAC learning with respect to concept classes whose members are approximated by fixed-degree polynomials. We adopt the Hilbert space representation in Section \ref{sub:HilbertSP} to analyze PAC learning using Algorithm \ref{alg:L2} and \ref{alg:fourier}. We start with the following formulation:
\begin{definition} 
Given $\epsilon\in [0,1]$, $k\in \NN$ and  any probability distribution $D_{\bfX}$ on $\mathcal{X}$, a concept class $\mathcal{C}$ of  functions $c:\mathcal{X}\mapsto \pmm$ is $(\epsilon,k)$-approximated if
$$\sup_{c\in \mathcal{C}}\inf_{p\in \Pk} \EE\big[\big(c(\bfX)-p(\bfX)\big)^2\big]\leq \epsilon^2,$$
where $\Pk$ is the set of all polynomials of degree up to $k$.   
\end{definition}

We consider agnostic PAC learning with respect to $\mathcal{C}$ and under the $0-1$ loss function. The minimum generalization error and empirical error of $\mathcal{C}$ are, respectively, defined as
\begin{align*}
\Pek &\deq \min_{c\in \mathcal{C}}\PP_{D}\big\{Y\neq c(\bfX)\big\},\\
  \Pekemp&\deq \min_{c\in \mathcal{C}}\Pemp_{\Demp}\big\{Y\neq c(\bfX)\big\}.
\end{align*}

We use the Hilbert space representation in Section \ref{sub:HilbertSP} and provide a lower bound on $\Pek$. 

\begin{lem}\label{lem:popt bound norm 1}
The minimum generalization error attainable by any $(\epsilon, k)$ concept class $\mathcal{C}$ is bounded from below as
\begin{align*}
\Pek \geq \frac{1}{2}-\frac{1}{2}\norm{{p}^*}_{1,D}-\epsilon, 
\end{align*}
where $p^* = \argmin_{p\in \Pk} \EE_{D}\big[\big(Y-p(\bfX)\big)^2\big]$.
\end{lem}

\begin{proof}
From \eqref{eq:error and inner prod} the $0-1$ loss of any function $c\in \mathcal{C}$ can be written as $\prob{Y\neq c(\bfX)}=\frac{1}{2}-\frac{1}{2}\<Y,c\>.$   Let $p\in \Pk$ be such that $\norm{c-p}_{2,D}\leq \epsilon$. 
Then, by adding and subtracting $p$, we obtain that
\begin{align}\nonumber
\<Y, c\> &= \<Y, p\> + \<Y, (c-p)\>\\\label{eq:Y,c}
&\leq \<Y, p\> + \norm{Y}_2 \norm{c-p}_2\leq \<Y, p\> + \epsilon,
\end{align}
where the first inequality follows from Cauchy–Schwarz inequality and the second inequality follows because $\norm{Y}_2=1$.  
Note that $\Pk$, the set of all polynomials on $\mathcal{X}$ 
with degree upto $k$, is a (finite dimensional) subspace inside the Hilbert space $\LD$. Therefore, it has an orthonormal basis denoted by $\{\Psi_1, \Psi_2, ..., \Psi_m\}$, where $m$ is less than $O(d^k)$. 
As a result, the polynomial $p$ can be written as $p\equiv \sum_{j=1}^m \<p,\Psi_j\> \Psi_j$. Hence, 
\begin{align*}
\<Y, p\> = \sum_{j=1}^m \<p,\Psi_j\>~ \<Y,\Psi_j\> = \<\Pi_Y, p\>,
\end{align*}
where $\Pi_Y \equiv \sum_{j=1}^m \<Y,\Psi_j\> \Psi_j$ is the \textit{projection} of $Y$ onto this subspace.  Consequently, from the above equality and \eqref{eq:Y,c}, we obtain that
\begin{align*}
\<Y, c\> &\leq \<\Pi_Y, p\> + \epsilon = \<\Pi_Y, c\> + \<\Pi_Y, (p-c)\>+ \epsilon\\
& \leq \<\Pi_Y, c\> + \norm{\Pi_Y}_2 \norm{p-c}_2+ \epsilon\\
&\leq \norm{\Pi_Y}_1 + \norm{\Pi_Y}_2 \norm{p-c}_2+ \epsilon\\
&\leq \norm{\Pi_Y}_1+2\epsilon,
\end{align*}
where the second inequality follows from Cauchy–Schwarz inequality, the third one holds as $|c(\bfx)|=1$ and the last inequality follows from Bessel's inequality, implying  $\norm{\Piy}_2\leq 1$, and the assumption that $\norm{p-c}_2\leq \epsilon$.
Next, we proceed with the following fact about the projection.
\begin{fact}\label{fact:Pi_Y}
$\Pi_Y$ the projection of $Y$ onto $\Pk$ is the polynomial minimizing $\EE\big[\big(Y-p(\bfX)\big)^2\big]$ over all $p\in \Pk$.
\end{fact}
The proof is complete by the following fact implying that $\Pi_Y\equiv p^*$.  
\end{proof}

We show in Section \ref{sub:Fourier PAC concentrated} that the lower-bound in Lemma \ref{lem:popt bound norm 1} helps to prove our results for the low-degree algorithm.

\subsection{PAC Learning Bounds}\label{sub:PAC concentrated}
Next, we analyze Algorithm \ref{alg:L2} and \ref{alg:fourier} for this class and prove the first main result of the paper.

\begin{theorem}\label{thm:L2 polynomials}
Given $\epsilon>0$ and $k\in \NN$, the  degree $k$ $\Ltwo$ polynomial regression agnostically PAC learns any $(\epsilon, k)$-approximated concept class with expected error up to 
 \begin{align*}
2\Pek +3\epsilon +\sqrt{\frac{2~d^{k+1}}{n}\log\frac{en}{d^{k+1}}},
\end{align*}
where $d$ is the number of input variables and $n$ is the sample size. 
\end{theorem}

\section{Proof of Theorem \ref{thm:L2 polynomials}}\label{proof:thm:L2 polynomials}

Let $\hat{p}$ be the output of $\Ltwo$-polynomial regression, that is
\begin{align*}
\hat{p}=\argmin_{p:\degree(p)\leq k} \frac{1}{nn}\sum_i \big(y_i-p(\bfx_i)\big)^2.
\end{align*}

\noindent\textbf{PAC bounds for the Fourier algorithm:}
Next, we employ a low-degree (Fourier) algorithm (Algorithm \ref{alg:fourier}) for PAC learning with respect to the polynomially approximated hypothesis class.

\begin{theorem}\label{thm:low degree concentrated}
Let $D$ be a joint probability distribution with marginal $D_X$ that is a product probability distribution on $\pmm^d$.  Then, for any $\delta\in[0,1]$, with probability at least $1-\delta$, the Fourier-based algorithm agnostically PAC learns any $(\epsilon, k)$-approximated concept class with generalization error up to 
 \begin{align}\label{eq:low deg error}
2\Pek+2\epsilon +O\Big(\sqrt{\frac{ d^k c_k}{(k-1)!n }\log\frac{4d^k}{(k-1)!\delta}}\Big),
\end{align}
where $c_k\deq \max_{\mathcal{S}\subseteq [d], |\mathcal{S}|\leq k} \norm{\pS}_{\infty}^2$. 
\end{theorem}
\begin{proof}
We prove the theorem by characterizing the effect of $2$-norm estimation on the error probability. Let
\begin{align*}
{p}^*=\argmin_{p\in \Pk} \norm{Y-p}^2_{2,D}.
\end{align*}
 From the second equality in \eqref{eq:error and inner prod}, the generalization error of $\hat{g}$ in Algorithm \ref{alg:fourier} satisfies 
\begin{align}\nonumber
\PP\Big\{Y\neq \hat{g}(\bfX)&\Big\}=\frac{1}{4}\norm{Y-\hat{g}}_{2, D}^2\\\nonumber
&\leq \frac{1}{4}\Big( \norm{Y-p^*}_{2, D}+\norm{p^*-\hat{g}}_{2, D}\Big)^2\\\label{eq:y_neq_ghat_D}
&\leq  \frac{1}{2}\Big(\norm{Y-{p}^*}_{2, D}^2 +\norm{p^*-\hat{g}}_{2, D}^2\Big),
\end{align}
where the  inequality follows from Minkowski's inequality for 2-norm. Observe that \eqref{eq:y_neq_ghat_D} is an upper bound on the generalization error in terms of $2$-norm quantities. Since $p^*$ minimizes the square loss, the first term in \eqref{eq:y_neq_ghat_D} equals
\begin{align*}
\norm{Y-{p}^*}_{2, D}^2 = 1-\norm{{p}^*}_{2}^2.
\end{align*}
We proceed by bounding the second term in \eqref{eq:y_neq_ghat_D}.  From Minkowski's inequality for $2$-norm and by adding and subtracting $\Piest$ as in Algorithm \ref{alg:fourier}, we have that
\begin{align}\nonumber
\norm{p^*-\hat{g}}^2_{2}&\leq \norm{p^*-\Piest}^2_{2}+\norm{\Piest - \hat{g}}^2_2\\\label{eq:pstar-ghat}
&+2\norm{p^*-\Piest}_{2}\norm{\Piest - \hat{g}}_2.
\end{align}
The first term in \eqref{eq:pstar-ghat} is bounded from Lemma \ref{lem:Pi_Y est}. As a result, $\norm{p^*-\Piest}_{2} = \epsilon'_n$, where $$\epsilon'_n = O\Big(\sqrt{\frac{ d^k c_k}{(k-1)!n }\log \frac{4d^k}{(k-1)!\delta}}\Big),$$ with probability at least $(1-\delta)$.  As for the second term in \eqref{eq:pstar-ghat}, we use the identity  $|h-\sign[h]|= |1-|h||$ for any function $h$. Therefore, as $\hat{g}=\sign[\Piest]$, we obtain that
\begin{align}\nonumber
\norm{\Piest - \hat{g}}_2^2&=\EE\left[(1-|\Piest(\bfX)|)^2\right]\\\label{eq:Pihat-ghat}
&= 1+\norm{\Piest}_{2}^2-2\norm{\Piest}_{1}.
\end{align}
Next, we show that the third term in \eqref{eq:pstar-ghat} is less than $4\epsilon'_n$. It suffices to show that $\norm{\Piest - \hat{g}}_2\leq 2$. For that, we use the equality in \eqref{eq:Pihat-ghat}.  By removing the last term in \eqref{eq:Pihat-ghat} and taking the square root we have  
\begin{align*}
\norm{\Piest - \hat{g}}_2 \leq  \sqrt{ 1+\norm{\Piest}_{2}^2}.
\end{align*}
 From the Minkowski's inequality we have that
\begin{align*}
\norm{\Piest}_2&\leq \norm{p^*}_2+\norm{p^*-\Piest}_2\\
& \leq  \norm{p^*}_2+\epsilon_n'\leq 1+\epsilon'_n.
\end{align*}
Hence, we get the desired bound $\norm{\Piest - \hat{g}}_2\leq \sqrt{1+(1+\epsilon'_n)^2}\leq 2$, assuming that $\epsilon'_n\leq 1/3$.
Combining the bounds for each term in \eqref{eq:pstar-ghat}, we get 
\begin{align*}
\norm{p^* - \hat{g}}_2^2&\leq {\epsilon'_n}^2+
1+\norm{\Piest}_{2}^2-2\norm{\Piest}_{1} + 4\epsilon'_n\\
&\leq 1+\norm{\Piest}_{2}^2-2\norm{\Piest}_{1} + 5\epsilon'_n.
\end{align*}
We plug this inequality in \eqref{eq:y_neq_ghat_D}. After rearranging the terms by adding and subtracting $\norm{p^*}_1$,  we obtain the following inequality
\begin{align*}
\PP\Big\{Y\neq \hat{g}&(\bfX)\Big\}\leq \frac{1}{2}\Big(2-2 \norm{p^*}_1+5\epsilon'_n\\
&+2\big(\norm{p^*}_1-\norm{\Piest}_1\big)+ \big(\norm{\Piest}_{2}^2 - \norm{p^*}_{2}^2\big)\Big)\\
&\leq 2\Pek+2\epsilon +5\epsilon'_n,
\end{align*}
where the last inequality follows from Lemma \ref{lem:popt bound norm 1} and the following argument for bounding the last two terms in the first inequality: 

For the $1$-norm difference, the Minkowski's inequality for $1$-norm gives
\begin{align*}
\norm{p^*}_1-\norm{\Piest}_1 \leq \norm{p^*-\Piest}_1 \leq \norm{p^*-\Piest}_2=\epsilon'_n,
\end{align*}
where the last inequality follows from the Jensen's inequality implying that $\norm{\cdot}_1\leq \norm{\cdot}_2$. 

For the difference of square of $2$-norms, we apply the Minkowski's inequality for $2$-norm and obtain
\begin{align*}
\norm{\Piest}_{2}^2 &\leq \norm{p^*}_{2}^2+\norm{p^*-\Piest}_{2}^2+2\norm{p^*}_{2}\norm{p^*-\Piest}_{2}\\
&\leq \norm{p^*}_{2}^2+3 \epsilon'_n.
\end{align*}
where the last inequality holds as $\norm{p^*}_{2}\leq 1$. 
%
\end{proof}

We end this section by presenting a simplified result of Theorem \ref{thm:low degree concentrated}.
\begin{corollary}\label{cor:Fourier PAC}
If the expected value of each $X_j$ satisfies $|\mu_j|\leq 1-\frac{1}{k}$, then the generalization error of the Fourier algorithm is upper bounded by
$$2\Pek+2\epsilon +O\Big(\sqrt{\frac{ \sqrt{k} (ed)^k}{n }\big(k \log \frac{e d}{k}+\log \frac{2\sqrt{k}}{\delta}\big)}\Big).$$
\end{corollary}

\begin{proof}

From the definition of $c_k$, we can write 
\begin{align*}
c_k&=\max_{\mathcal{S}: |\mathcal{S}|\leq k} \max_{\bfx \in \pmm^d} |\ps(\bfx)|^2\leq \max_{\mathcal{S}: |\mathcal{S}|\leq k} \prod_{j\in\mathcal{S}}\frac{(1+|\mu_j|)^2}{\sigma^2_j}\\
&= \max_{\mathcal{S}: |\mathcal{S}|\leq k} \prod_{j\in\mathcal{S}}\frac{1}{{1-|\mu_j|}},
\end{align*}
where the first inequality is from the definition of $\ps$. The laste quality holds as $\sigma_j^2=1-\mu_j^2$. Therefore, under the assumption that  $|\mu_j|\leq 1-\frac{1}{k}$, the following inequality holds
\begin{align*}
c_k \leq  \max_{\mathcal{S}: |\mathcal{S}|\leq k}  \prod_{j\in\mathcal{S}}{k}\leq k^{k}
\end{align*}
Hence, 
\begin{align*}
\frac{c_k}{(k-1)!}\leq \frac{k k^k}{k!}=2^{(k+1)\log_2 k -\log_2 k!}
\end{align*}
From Stirling approximation $\log_2 k! = k\log_2 k - k\log_2 e +O(\log_2 k)$. Hence, 
\begin{align*}
\frac{c_k}{(k-1)!}\leq 2^{k\log_2 e+ O(\log_2 k)}=e^k+O(k).
\end{align*}
Using the above inequality and Theorem \ref{thm:low degree concentrated} in the main text, we obtain the corollary. 

\end{proof}

\section{Learning Other Hypothesis Classes}\label{sec:other classes}
In this section, we extend our results to two other type of concept classes. The first one is called \textit{half-spaces} and the other one is a generalized version of the concentrated hypothesis classes.
\subsection{Half-spaces}
In this section, we consider learning another class of functions called half-spaces.  More precisely, a half-space a Boolean-valued function of the form $$c(\bfx)=\sign[a_0+\sum_{j=1}^d a_j x_j)],\quad \forall \bfx \in \RR^d$$ where $a_j\in \RR$. We start with a lower-bound on the optimal classification error of the class.

\begin{lem}
Let $D$ be any joint probability distribution on $\RR^d\times \pmm$ with marginal $\Dx$ that is the uniform distribution on $\SS^{d-1}$ or jointly Gaussian on $\RR^d$. Then, for any $\epsilon >0$, the minimum generalization error of learning with respect to half-spaces satisfy the following lower bound
\begin{align*}
\Pek \geq \frac{1}{2} - \frac{ 1}{2}\norm{p_\epsilon^*}_{1, D_X} - \epsilon,
\end{align*}
where $p_\epsilon^*$ is a polynomial of degree up to $O(\frac{1}{\epsilon^4})$ minimizing $\norm{Y-p}_{2,D}$ among all such polynomials.
\end{lem}
The proof of the lemma follows from Lemma \ref{lem:popt bound norm 1} and  \citet{Kalai2005}'s result (Theorem 6) on the sign function. This result is stated as 
\begin{lem}[\citet{Kalai2005}]
Let $X$ be a random variable with uniform distribution on $\SS^{d-1}$ or jointly Gaussian on $\RR^d$. Then, for any $\epsilon>0$, there exists a polynomial $p$ of degree  $O(\frac{1}{\epsilon^4})$ such that $\EE\Big[\big(p(\bfX) - \sign(\bfX)\big)^2 \Big]\leq \epsilon^2.$
\end{lem}
This lemma makes a connection between half-spaces and the polynomial-approximated class. That said, in the following theorems we show our results for PAC learning using Algorithm \ref{alg:L2}.
\begin{theorem}\label{thm:L2 polynomials half-spaces}
Let $D$ be any joint probability distribution on $\RR^{d}\times \pmm$, with marginal $D_X$ that is uniform on the unit sphere or jointly Gaussian. Then, $\Ltwo$-polynomial regression PAC learns half-spaces with expected generalization error up to 
 \begin{align*}
2\Pek +3\epsilon  +\sqrt{\frac{d^{O(\frac{1}{\epsilon^4})}}{n}\log\frac{n}{d^{O(\frac{1}{\epsilon^4})}}}.
\end{align*}

\end{theorem}

\subsection{Generalized approximated class}
Lastly, we finish this paper by extending our results to a more general hypothesis class. Fix a set of functions $e_1(\bfx), e_2(\bfx), ..., e_m(\bfx)$ and let $\mathcal{H}$ be a Hilbert space spanned by a these functions. Let $\mathcal{C}$ be a class of functions each of which approximated by elements of $\mathcal{H}$  with square error up to $\epsilon$, that is, $$\inf_{h\in \mathcal{H}}\norm{c-h}_{2,D}\leq \epsilon,$$ for any $c\in \mathcal{C}$.  As a special case, suppose $e_i$'s are all the functions of the form $e(\bfx)=\prod_{j\in [d]}x^{\alpha_j}_j$ where $\alpha_j$'s are non-negative integers adding up to $k$. Then  $\mathcal{C}$ is a $(k,\epsilon)$-approximated class as in Section \ref{sec:poly approx}. 

\begin{theorem}\label{thm:general hilbert class}
Suppose $A$ is any algorithm that given $n$ training instances finds a function $\hat{h}\in \mathcal{H}$ so that the empirical loss $\norm{Y-h}_{2,\Demp}$ is minimized. Then, the predictor $\sign[\hat{h}]$ learns $\mathcal{C}$ with expected generalization error up to 
 \begin{align*}
 2\Pek+3\epsilon+O\Big(\sqrt{\frac{\emph{VC}(\mathcal{C})}{n}\log\frac{n}{\emph{VC}(\mathcal{C})}}\Big),
\end{align*}
where $\emph{VC}(\mathcal{C})$ is the \ac{VC} dimension of $\mathcal{C}$. 
\end{theorem}

\section*{Acknowledgement}

This work was supported in part by
NSF Center on Science of Information
Grants CCF-0939370 and NSF Grants CCF-1524312, CCF-2006440, CCF-2007238, and Google Research Award.

\appendices 
\section{Proof of Lemma \ref{lem:Pi_Y est}}\label{proof:Pi_Y est}
\paragraph{Mean and variance estimations:}
We first take into account the effect of the imperfections in mean and variance estimation. For tractability of our analysis, we use a fraction of the  training samples just for the mean and variance estimations. As a measure of accuracy of the estimations, we require the differences $|\hat{\mu}_j-\mu_j|$ and $|1- \frac{\sigma}{\hat{\sigma}}|$ to be sufficiently small with probability close to one. This is a deviation from standard measures of estimations in which the variance of the differences are required to be small. In the following lemma, we bound the estimation errors in terms of the number of the samples.

\begin{lem}\label{lem:mean_var_error}
Given  $\epsilon_0,\delta_0\in (0,1)$ the following inequalities hold with probability at least $(1-\delta_0)$
\begin{align}\label{eq:mean_var error}
\big| \hat{\mu}_j - \mu_j \big| \leq \epsilon_0,\qquad \qquad \big|1- \frac{\sigma_j}{\hat{\sigma}_j} \big| \leq \frac{2\epsilon_0}{\sigma^2_j},
\end{align}
for all $j\in [d]$, provided that atleast $n_0(\epsilon_0, \delta_0) = \frac{2}{\epsilon_0^2}\log\frac{2d}{\delta_0}$ samples are available.
\end{lem}

\begin{proof}
Form McDiarmid's inequality, for each $j\in [d]$ we have 
\begin{align*}
\PP\{|\hat{\mu}_j-\mu_j| \geq\epsilon_0 \}&\leq 2\exp\{-\frac{ n\epsilon_0^2}{2}\}.
\end{align*}
Therefore, applying the union bound gives 
\begin{align*}
\PP\Big\{\bigcup_{j=1}^d \big\{|\hat{\mu}_j-\mu_j| \geq \epsilon_0\big\} \Big\}&\leq  2d\exp\{-\frac{n\epsilon_0^2}{2}\}.
\end{align*}
Thus, the right-hand side of the above inequality is less than $\delta_0$, if $
n \geq \frac{2}{\epsilon_0^2}\log(\frac{2d}{\delta_0})$.
As a result we obtain the inequalities for the estimation of $\mu_j$'s. Next, we prove the inequalities for the estimation of $\sigma_j$'s. For any fixed $\hat{\mu}\in (-1,1)$, define the function $h_{\hat{\mu}}(x)=\frac{\sqrt{1-x^2}}{\sqrt{1-\hat{\mu}^2}}$. From Taylor's theorem, there exists $\zeta\in (-1,1)$ which is between $x$ and $\hat{\mu}$ such that 
  \begin{align*}
  h_{\hat{\mu}}(x)=1-\frac{\zeta(x-\hat{\mu})}{\sqrt{(1-\zeta^2)(1-\hat{\mu}^2)}}.
  \end{align*}
  As a result,
  \begin{align*}
   |h_{\hat{\mu}}(x)-1| =\frac{|\zeta||x-\hat{\mu}|}{\sqrt{(1-\zeta^2)(1-\hat{\mu}^2)}} \leq \frac{|x-\hat{\mu}|}{\sqrt{(1-(\max\{x,\hat{\mu}\})^2)(1-\hat{\mu}^2)}}. 
   \end{align*} 
  Now by setting $x=\mu_j$ and that $|\hat{\mu}_j-\mu_j|\leq \epsilon_0$, we have 
   \begin{align*}
   |\frac{{\sigma}_j}{\hat{\sigma}_j}-1| =  |h_{\hat{\mu}}(\mu)-1|\leq \frac{\epsilon_0}{\hat{\sigma} \min\{\hat{\sigma},\sigma\}}. 
   \end{align*}  
   Note that, $|\hat{\mu}_j|\leq |\mu_j|+\epsilon_0$. Therefore, $$\hat{\sigma}_j^2 \geq 1-(|\mu_j|+\epsilon_0)^2 \geq \sigma_j^2 - 2\epsilon_0 |\mu_j| -\epsilon_0^2\geq \sigma_j^2 -3\epsilon_0. $$
   As a result, 
   \begin{align*}
   |\frac{{\sigma}_j}{\hat{\sigma}_j}-1| \leq \frac{\epsilon_0}{ \sigma_j^2 -3\epsilon_0}\leq \frac{2 \epsilon_0}{ \sigma_j^2}, 
   \end{align*}
   which completes the proof of the lemma.
\end{proof}
Now we proceed with the proof of the lemma. Let $\Pibar$ denote the version of $\Piest$ under the assumption that $\hat{\mu}_j=\mu_j$ and $\hat{\sigma}_j=\sigma_j$ for all $j\in [d]$. Also, let $B$ be the even that the inequalities in \eqref{eq:mean_var error} hold.   From Minkowsky's inequality, by adding and subtracting $\Pibar$ we have
\begin{align*}
\norm{\Piy-\Piest}_2 \leq \underbrace{\norm{\Piy-\Pibar}_2}_{V}+\underbrace{\norm{\Pibar-\Piest}_2}_{W}.
\end{align*}
Let $V$ and $W$ denote the first and the second term above, respectively. We proceed by the following lemmas.

\begin{lem}\label{lem:fbarJbar-fJbar} 
Given any $\delta>0$, the inequality 
$\norm{\Piy-\Pibar}_2 \leq \sqrt{\frac{2 d^k c_k}{(k-1)!n }\log \frac{2d^k}{(k-1)!\delta}}$ holds with probability $(1-\delta)$.
\end{lem}

\begin{proof}
Recall that $\Pibar$ is defined as 
$$\Pibar(x^d)\deq \sum_{\mathcal{S}: |\mathcal{S}|\leq k} \fbarS \pS(x^d),$$ 
where the Fourier-estimates $ \fbarS$ are defined as $\fbarS \deq \frac{1}{n}\sum_{i} Y(i) \pS(X(i)).$ In addition, by definition of the projection function $\Piy$, we have 
\begin{align*}
\Piy(\bfx)= \sum_{\mathcal{S}: |\mathcal{S}|\leq k} \fS~\pS(\bfx), \qquad \forall \bfx\in \mathcal{X}^d.
\end{align*}
Therefore, from Parseval's identity, the $2$-norm factors as 
\begin{align*}
\norm{\Piy-\Pibar}_2^2&=\sum_{\mathcal{S}: |\mathcal{S}|\leq k} |\fS-\fbarS|^2.
\end{align*}
In what follows, we show that $|\fS-\fbarS|\leq \epsilon$ for all subsets $\mathcal{S}\subseteq [d]$ with $|\mathcal{S}|\leq k$. Note that $\fbarS$ is a function of the training random samples $(X(i), Y(i)), i=1,2,...,n$. Observe that $\EE[\fbarS]=\fS$ which implies that $\fbarS$ is an unbiased estimation of $\fS$. Since the samples are drawn \ac{IID},  we apply McDiarmid's inequality to bound the probability of the event  $|\fS-\fbarS|\geq \epsilon'$. 

For that, fix $i\in [d]$ and suppose $(\bfX(i), Y(i))$ in the training set is replaced with an \ac{IID} copy $(\tilde{\bfX}(i), \tilde{Y}(i))$. With this replacement $\fbarS$ is changed to another random variable denoted by $\tilde{f}_{ S} $. Then 
\begin{align*}
 |\fbarS-\tilde{f}_{ S} |&= \frac{1}{n} |Y(i)\pS({\bfX}(i))-\tilde{Y}(i)\pS(\tilde{\bfX}(i))|\\
 &\leq \frac{1}{n} |Y(i)\pS({\bfX}(i))|+|\tilde{Y}(i)\pS(\tilde{\bfX}(i))|\\
  &\leq \frac{1}{n} |\pS({\bfX}(i))|+|\pS(\tilde{\bfX}(i))|\\
 &\leq \frac{2}{n} \norm{\pS}_{\infty},
 \end{align*} 
 where $ \norm{\psi_S}_{\infty}=\max_{\bfx}|\psi_S(\bfx)|$. Let $c_k=\max_{\mathcal{S}\subseteq [d], |\mathcal{S}|\leq k} \norm{\pS}_{\infty}^2 $. Then, from McDiarmid's inequality, for any $\epsilon' \in (0,1)$ 
\begin{align}\label{eq:chernoff_S} 
\PP\Big\{ \max_{\mathcal{S}: |\mathcal{S}|\leq k} \big|\fbarS-\fS \big|\geq \epsilon' \Big\}\leq 2\Big[\sum_{m=0}^k \binom{d}{m} \Big] \exp\big\{-\frac{n\epsilon'^2}{2 c_k}\big\},
\end{align}
where we also used the union bound. For $k\leq d/2$, we obtain that 
$$\sum_{m=0}^k \binom{d}{m}\leq  k \frac{d^k}{k!}.$$
As a result, with probability at least $(1-\delta)$, $\max_{\mathcal{S}: |\mathcal{S}|\leq k} \big|\fbarS-\fS \big| \leq \sqrt{\frac{2 c_k}{n }\log \frac{2d^k}{(k-1)!\delta}}$.
Hence, we with probability at least $(1-\delta)$
\begin{align*}
\norm{\Piy-\Pibar}_2^2\leq \frac{2 d^k c_k}{(k-1)!n }\log \frac{2d^k}{(k-1)!\delta},
\end{align*}
and the proof is complete by taking the square root of both sides.
\end{proof}

\begin{lem}\label{lem:mean_est_eff}
Conditioned on $B$, the inequalities $
\norm{\Pibar-\Piest}_\infty\leq \lambda(\epsilon)$ hold, almost surely, for all $k$-element subsets $\mathcal{J}\subset [d]$, where $\lambda$ is a function satisfying  $\lambda(\epsilon_0)=   O(\frac{k d^k c_k}{(k-1)!} \epsilon_0)$ as $\epsilon_0\rightarrow 0$.
\end{lem} 

Recall that the function $\Pibar$ is defined as 
$$\Pibar(x^d)\deq \sum_{\mathcal{S}: |\mathcal{S}|\leq k} \fbarS \pS(x^d),$$ 
where the Fourier-estimates $ \fbarS$ are defined as $$ \fbarS \deq \frac{1}{n}\sum_{i} Y(i) \pS(X(i)).$$
From triangle inequality for $\infty$-norm and the definition of $\Piest$ and $\Pibar$ we obtain 
\begin{align}\label{eq:fhatJ fbarJ}
 \norm{\Piest-\Pibar}_\infty &\leq \sum_{\mathcal{S}: |\mathcal{S}|\leq k} \norm{\festS ~\pestS - \fbarS~ {\psi}_S  }_\infty.
 \end{align} 
 Again by triangle inequality and by adding and subtracting $\fbarS \pestS$, we obtain that 
 \begin{align*}
 \norm{\festS ~\pestS - \fbarS~ {\psi}_S  }_\infty &\leq \norm{\festS ~\pestS - \fbarS ~\pestS  }_\infty +  \norm{\fbarS ~\pestS - \fbarS~ {\psi}_S  }_\infty\\
  &= |\festS-\fbarS|~\norm{\pestS}_\infty + |\fbarS|~ \norm{\pestS - {\psi}_S  }_\infty.
 \end{align*}
 Next, note that from triangle inequality
\begin{align*}
|\festS-\fbarS|\leq \frac{1}{n}\sum_{i}|\pestS(\mathbf{x}(i))-\pS(\mathbf{x}(i))|\leq \norm{\pS-\pestS}_{\infty}.
\end{align*}
Therefore,  
\begin{align}\label{eq:fhatS phihatS}
 \norm{\festS ~\pestS - \fbarS~ {\psi}_S  }_\infty &\leq \big( \norm{\pestS}_\infty + |\fbarS|\big) \norm{\pestS - {\psi}_S  }_\infty.
\end{align}
 We proceed by bounding each term above. 
 As for the first term we have, that $\norm{\pestS}_\infty\leq \norm{\pS}_\infty+\norm{\pestS-\pS}_\infty$. 
%
 As for the second term, we have 
 \begin{align*}
 \fbarS = \frac{1}{n}\sum_i Y(i) \pS(\bfX(i)) \leq \norm{\pS}_\infty.
 \end{align*}
 Lastly,  the third term is bounded using the following lemma.
\begin{lem}\label{lem:pS and Pshat norm infty}
Conditioned on $B$, the inequality $\norm{\pS-\pestS}_{\infty}\leq \gamma(\epsilon_0)$ holds, almost surely, where $\gamma$ is a function satisfying $\gamma(\epsilon_0) = O(k\epsilon_0\sqrt{c_k})$ as $\epsilon_0\rightarrow 0$.
\end{lem} 
 Before proving this lemma, we complete our argument. As a result of this lemma and using the triangle inequality, we obtain from \eqref{eq:fhatS phihatS} that 
\begin{align*}
\norm{\festS ~\pestS - \fbarS~ {\psi}_S  }_\infty &\leq \big(2 \norm{\pS}_\infty + \norm{\pestS-\pS}_\infty  \big) \norm{\pestS - {\psi}_S  }_\infty\\
&\leq   \big( 2\sqrt{c_k}+\gamma(\epsilon_0)\big) \gamma(\epsilon_0). 
\end{align*}
Lastly, from \eqref{eq:fhatJ fbarJ} we get the following bound
\begin{align*}
\norm{\Piest-\Pibar}_\infty&\leq \lambda(\epsilon_0)\deq \frac{d^k}{(k-1)!}\big( 2\sqrt{c_k} \gamma(\epsilon_0) + \gamma^2(\epsilon_0)\big).
\end{align*}
It is not difficult to check that $\lambda(\epsilon_0) = O(\frac{k d^k c_k}{(k-1)!} \epsilon_0)$ as $\epsilon_0\rightarrow 0$. Now it remains to prove Lemma \ref{lem:pS and Pshat norm infty} which is given below:

{\noindent  \bf Proof of Lemma \ref{lem:pS and Pshat norm infty}:}
 We start with the triangle inequality for $\infty$-norm by adding and subtracting $b_{\mathcal{S}}\pS$:
\begin{align*}
\norm{\pS-\pestS}_{\infty}&\leq \norm{\pS-b_{\mathcal{S}}\pS}_{\infty}+\norm{b_{\mathcal{S}}\pS-\pestS}_{\infty}.
\end{align*}
Note that $b_{\mathcal{S}}\pS\equiv \prod_{j\in \mathcal{S}}\frac{x_j-{\mu}_j}{\hat{\sigma}_i}$. Now, using the triangle inequality on the second term above, we have 
\begin{align*}
\norm{b_{\mathcal{S}}\pS-\pestS}_{\infty}&=  \norm{b_{\mathcal{S}}\pS  \pm  \big(\sum_{l\in \mathcal{S}}  \prod_{j\leq l} \frac{x_j-\hat{\mu}_j}{\hat{\sigma}_i} \prod_{r> l} \frac{x_r-\mu_r}{\hat{\sigma}_r}  \big)-\pestS}_{\infty} \\
&\leq  \sum_{l\in \mathcal{S}} \frac{|\mu_l-\hat{\mu}_l|}{\hat{\sigma}_l}~  \norm{\prod_{j< l} \frac{(x_j-\hat{\mu}_j)}{\hat{\sigma}_j} \prod_{r> l} \frac{(x_r-\mu_r)}{\hat{\sigma}_r}}_{\infty}\\
&\leq   \frac{\epsilon}{\sigma_{\min}}\sum_{l\in \mathcal{S}}~  \norm{\prod_{j< l} \frac{(x_j-\hat{\mu}_j)}{\hat{\sigma}_j} \prod_{r> l} \frac{(x_r-\mu_r)}{\hat{\sigma}_r}}_{\infty}\\
&\leq   \frac{\epsilon}{\sigma_{\min}}\sum_{l\in \mathcal{S}}~  \prod_{j< l} \frac{(1+|\hat{\mu}_j|)}{\hat{\sigma}_j} \prod_{r> l} \frac{(1+|\mu_r|)}{\hat{\sigma}_r}\\
&\stackrel{(a)}{\leq}   \frac{\epsilon}{\sigma_{\min}}\sum_{l\in \mathcal{S}}~  \prod_{j< l} \frac{(1+|{\mu}_j|)(1+\epsilon)}{\hat{\sigma}_j} \prod_{r> l} \frac{(1+|\mu_r|)}{\hat{\sigma}_r}\\
&\stackrel{(b)}{\leq}   \frac{\epsilon}{\sigma_{\min}}b_{\mathcal{S}} \sum_{l\in \mathcal{S}}~  \prod_{j\in \mathcal{S}} \frac{(1+|{\mu}_j|)(1+\epsilon)}{{\sigma}_j}\\
&\stackrel{(c)}{\leq} \frac{k\epsilon}{\sigma_{\min}} b_{\mathcal{S}}  (1+\epsilon)^k \norm{\pS}_{\infty},
 \end{align*} 
 where $(a)$ follows from the inequality $(1+|\hat{\mu}_j|)\leq (1+|\mu_j|)(1+\epsilon)$, and $(b)$ follows from $(1+|{\mu}_j|)\leq (1+|\mu_j|)(1+\epsilon)$. Lastly, $(c)$ holds as $|\mathcal{S}|\leq k$ and  because $\norm{\pS}_{\infty} = \prod_{j\in \mathcal{S}} \frac{1+|\mu_j|}{\sigma_j}$. 
 \begin{align}\label{eq:bound on infty norm}
\norm{\pS-\pestS}_{\infty}&\leq |1-b_{\mathcal{S}}| \norm{\pS}_{\infty} + \frac{k\epsilon}{\sigma_{\min}} b_{\mathcal{S}}  (1+\epsilon)^k \norm{\pS}_{\infty}.
 \end{align}
 From the assumption of the lemma and the definition of $b_{\mathcal{S}}$ we obtain that
 \begin{align*}
 1-(1+\epsilon)^{|S|}\leq            1- b_{\mathcal{S}}\leq  1-(1-\epsilon)^{|S|}.
 \end{align*}
 Since $\epsilon \in (0,1)$ and $|S|\leq k$, then $(1-\epsilon)^{|S|}\geq 1-k\epsilon $. Also, from the fact that $(1+x)\leq e^x$ for all $x\in \RR$, we obtain
  \begin{align}\label{eq:bound on c_S}
1-e^{k\epsilon}\leq  1- b_{\mathcal{S}}\leq  k\epsilon \leq e^{k\epsilon}-1.
 \end{align}
Lastly, combining \eqref{eq:bound on infty norm} and  \eqref{eq:bound on c_S} gives the following inequality
\begin{align*}
\norm{\pS-\pestS}_{\infty}&\leq (e^{k\epsilon}-1) \norm{\pS}_{\infty} + \frac{k\epsilon}{\sigma_{\min}} (1+\epsilon)^{2k}\norm{\pS}_{\infty}.
 \end{align*} 
 The proof is complete by noting that $\norm{\pS}_{\infty} \leq \sqrt{c_k}$.
\hfill $\blacksquare$

From Lemma \ref{lem:mean_est_eff}, we know that $W$ is measurable with respect to $B$. In particular, conditioned on $B$, $W\leq \lambda(\epsilon_0)$. Therefore, from the above lemmas and using the inequality $\norm{\cdot}_2\leq \norm{\cdot}_\infty$, we have, with probability $(1-\delta_0)(1-\delta)$ that 
\begin{align*}
\norm{\Piy-\Pibar}_2 \leq \sqrt{\frac{2 d^k c_k}{(k-1)!n }\log \frac{2d^k}{(k-1)!\delta}} + \lambda(\epsilon_0). 
\end{align*}
Now set $\epsilon_0 = \sqrt{\frac{2}{n_0}\log\frac{2d}{\delta}}$ with $\delta_0=\delta$. Then, with $n_0 = O(n)$, we get with probability $(1-\delta)^2$ that
$\norm{\Piy-\Pibar}_2 = O\Big(\sqrt{\frac{2 d^k c_k}{(k-1)!n }\log \frac{2d^k}{(k-1)!\delta}}\Big).$
Now the proof is complete by changing $\delta$ to $\delta/2$ and noting that $(1-\delta/2)^2 \geq 1-\delta$.
\qed

\bibliographystyle{IEEEtran}
\bibliography{References}
\end{document}